\documentclass{sig-alternate-per-modified}

\usepackage[a4paper,margin=2cm]{geometry}
\usepackage[T1]{fontenc}
\usepackage[utf8]{inputenc}

\usepackage{url}
\usepackage{balance}
\usepackage{graphicx}
\usepackage{amsmath, amssymb}

\usepackage{algorithm}
\usepackage{algpseudocode}

\usepackage{mathtools}

\usepackage{amssymb}

\usepackage{dsfont}    

\usepackage{tikz}
\usetikzlibrary{arrows.meta,positioning}
\usetikzlibrary{positioning,calc}

\usepackage{mathtools, nccmath,amsmath,amssymb,amsfonts,amstext, dsfont, color}

\usepackage{tikz}
\usetikzlibrary{automata,positioning,arrows}

\usepackage{booktabs}

\usepackage{subcaption}
\usepackage{mathtools, stmaryrd}
\usepackage{xparse} 
\usepackage{mathtools, stmaryrd}
\usepackage{algorithm}
\usepackage{algorithmicx}
\usepackage{algpseudocode}
\usepackage{nicefrac}
\usepackage{bm}
\usepackage{hyperref}
\usepackage{multirow}

\usepackage{graphicx} 

\newtheorem{claim}{Claim}[section]
\newtheorem{theorem}{Theorem}[section]
\newtheorem{corollary}{Corollary}[section]

\newcommand{\LFU}{\texttt{LFU}}
\newcommand{\LRU}{\texttt{LRU}}
\newcommand{\LeCar}{\texttt{LeCaR}}
\newcommand{\MC}{\texttt{MC}}
\newcommand{\Hedge}{\texttt{Hedge}}
\newcommand{\HMC}{\texttt{H-MC}}
\newcommand{\Proba}[1]{\mathrm{Pr}\left(#1\right)}

\newcommand{\E}[1]{\mathbb{E}\left[#1\right]}

\def\b0{\mathbf{0}}

\def\cA{\mathcal{A}}

\def\cE{\mathcal{E}}
\def\cO{\mathcal{O}}

\def\P{\mathbb{P}}

\usepackage{xcolor}

\definecolor{lightgray}{gray}{0.9}

\usepackage{mdframed} 

\newmdtheoremenv[
  linecolor=black,
  backgroundcolor=lightgray,
  linewidth=1pt,
  innerleftmargin=5pt,
  innerrightmargin=5pt,
  innertopmargin=1pt,
  innerbottommargin=1pt
]{theoremSp}{Theorem}

\newmdtheoremenv[
  linecolor=black,
  backgroundcolor=lightgray,
  linewidth=1pt,
  innerleftmargin=5pt,
  innerrightmargin=5pt,
  innertopmargin=1pt,
  innerbottommargin=1pt
]{lemmaSp}{Lemma}

\newmdtheoremenv[
  linecolor=black,
  backgroundcolor=lightgray,
  linewidth=1pt,
  innerleftmargin=5pt,
  innerrightmargin=5pt,
  innertopmargin=1pt,
  innerbottommargin=1pt
]{corollarySp}{Corollary}

\newmdtheoremenv[
  linecolor=black,
  backgroundcolor=lightgray,
  linewidth=1pt,
  innerleftmargin=5pt,
  innerrightmargin=5pt,
  innertopmargin=1pt,
  innerbottommargin=1pt
]{claimSp}{Claim}
\begin{document}

\title{No-Regret Mixing of LRU and LFU with Optimal Switching Cost}

\numberofauthors{2}

\author{
\alignauthor
Younes Ben Mazziane%
\titlenote{\scriptsize Both authors contributed equally to this work.}\\
\affaddr{Université d'Avignon}
\alignauthor
Xinying Zou\raisebox{9pt}{$\ast$}\\
\affaddr{Link{\"o}ping University}
}


\maketitle

\begin{abstract}

Caching systems often rely on simple eviction policies such as Least Recently Used (\texttt{LRU}) and Least Frequently Used (\texttt{LFU}), which perform well in complementary request regimes. Recent policies such as \texttt{LeCar} and \texttt{Cacheus} combine \texttt{LRU} and \texttt{LFU} using ideas from the experts problem in online learning. Specifically, upon a miss, they randomize between the two eviction rules using probabilities derived from scores updated by tracking the history of past evictions. While these policies exhibit strong empirical performance, it remains unclear whether they are guaranteed, on every request sequence, to perform asymptotically as well as the better of \texttt{LRU} and \texttt{LFU}, i.e., whether they achieve sublinear regret with respect to this benchmark. We first show that \texttt{LeCar} suffers linear regret against an oblivious adversary, even with unbounded history. We then propose \texttt{H-MC}, a \texttt{Hedge}-based mixture of virtual \texttt{LRU} and \texttt{LFU} caches that preserves \texttt{Hedge}'s selection probabilities, and hence its regret guarantees, while minimizing the switching cost among all joint selection rules with these marginals. Extensive numerical simulations show that \texttt{H-MC} outperforms \texttt{LeCar} on some synthetic traces.

\end{abstract}

\keywords{Caching, Expert's problem, Lazy online learning}

\section{Introduction}

Caching is widely used across computer systems, from improving CPU performance to enhancing user experience in content delivery networks. Its main goal is to decide which files should be stored locally so as to reduce the fraction of requests that cannot be served directly from the cache. \textit{Least Recently Used} ($\LRU$) and \textit{Least Frequently Used} ($\LFU$) are very popular caching policies. Upon a \textit{miss} (a request for an item not stored in the cache), $\LRU$ evicts the least recently requested item from the cache, whereas $\LFU$ evicts the least frequently requested one.

$\LRU$ and $\LFU$ are known to perform well in complementary regimes. For instance, $\LFU$ is asymptotically optimal when requests are drawn from a stationary distribution. In contrast, $\LRU$ is often more effective in non-stationary environments, where requests exhibit \textit{temporal locality}~\cite{traverso_temporal_2013}, such as when items are requested in bursts and may not be requested again afterward. Motivated by this complementarity, many subsequent caching policies have been proposed to adaptively combine recency-based and frequency-based decisions depending on the observed requests~\cite{donghee_lee_lrfu_2001,jiang_lirs_2002,megiddo_arc_2003,vietri_driving_2018,rodriguez_learning_2021}. Despite their strong empirical performance, these algorithms are not known to guarantee performance comparable to the better of $\LRU$ and $\LFU$ on arbitrary request sequences.

Worst-case guarantees for caching algorithms were first studied in the paging problem, where each miss triggers an eviction and hits leave the cache content unchanged. The standard metric is the competitive ratio, defined as the worst-case ratio between the miss count of the online policy and that of Belady's offline optimal algorithm. For cache capacity $C$, no deterministic online algorithm can achieve a competitive ratio smaller than~$C$, and several policies, including~$\LRU$ and~\texttt{FIFO}, attain this bound~\cite{boyar_relative_2007}. Thus, the competitive ratio can fail to distinguish algorithms with different practical behavior.

More recently, several caching policies based on online convex optimization~\cite{hazan_introduction_2016} have been proposed~\cite{paschos_learning_2019,bhattacharjee_fundamental_2020}. These policies are evaluated through regret, defined as the worst-case gap between the miss count of the algorithm and that of the optimal offline static caching policy, i.e., the cache storing the top-$C$ most requested items. Sublinear regret then guarantees that, asymptotically, the miss ratio of the policy approaches that of this optimal static cache. However, this benchmark can be weak, as it is frequency-based and therefore close in spirit to $\LFU$. 


The limitation is particularly clear for workloads with strong \textit{temporal locality}. For example, consider a cache of size $C=2$ and a trace composed of $n$ phases, where phase $i$ consists of $2r$ requests of the form $(a_i,b_i)^r$, with $a_i$ and $b_i$ being fresh items. Since every item appears exactly $r$ times, storing any two distinct items is optimal among static caches and incurs $2(n-1)r$ misses. In contrast, $\LRU$ incurs $2n$ misses, which is much smaller than the optimal static offline benchmark when $r$ is large. This shows that no-regret guarantees with respect to the optimal static cache do not necessarily imply good performance on recency-friendly workloads.

Along the same line of work,~\cite{daniely_competitive_2019} proposed caching policies that are simultaneously competitive with Belady's offline optimal algorithm, and achieve sublinear regret with respect to the optimal static caching policy. However, as discussed above, worst-case competitive guarantees can be overly pessimistic, while a static-cache benchmark can be weak for workloads with strong temporal locality.

In this paper, we consider as performance metric, the worst-case gap, over all request sequences of length~$T$, between the number of misses incurred by an algorithm $\cA$ and the smaller miss count incurred by $\LRU$ or $\LFU$. We refer to this metric throughout the paper simply as \textit{regret}, and we denote it as~$\mathcal{R}_{T}^{\cA}$. The objective is then to design efficient caching policies with sublinear regret, i.e., $\mathcal{R}_{T}^{\cA}=o(T)$. Such policies are then guaranteed to adapt to both frequency and recency driven workloads.

A natural way to achieve sublinear regret is to formulate caching as an experts problem~\cite{cesa-bianchi_prediction_2006}, with $\LRU$ and $\LFU$ as experts. The policies $\LeCar$~\cite{vietri_driving_2018} and \texttt{Cacheus}~\cite{rodriguez_learning_2021} follow this principle: upon a miss, they probabilistically choose whether to evict according to~$\LRU$ or~$\LFU$. Their selection probabilities are inspired by $\Hedge$~\cite{cesa-bianchi_prediction_2006}; they maintain scores for the two eviction rules and sample from the corresponding softmax distribution. The scores are updated through a history of past evictions. When an evicted item is requested again, the rule responsible for its eviction is penalized. However, it remains unclear if such policies inherit $\Hedge$'s sublinear regret guarantees. We thus aim to address the following questions:  
\begin{center}
\begin{minipage}{0.9\linewidth}
\itshape
\begin{enumerate}
    \item Do existing online-learning inspired caching policies, such as $\LeCar$, have sublinear regret guarantees?

    \item If not, what are efficient alternative ways to combine $\LRU$ and $\LFU$ to achieve such guarantees?    
\end{enumerate}
\end{minipage}
\end{center}

\subsection*{Contributions}
This paper makes two main contributions. First, we prove that $\LeCar$ approach to combining $\LRU$ and $\LFU$ exhibits linear regret against an oblivious adversary, i.e., $\mathcal{R}_{T}^{\LeCar}=\Omega(T)$, even when the history size is unbounded. Second, we propose \texttt{H-MC}, a \texttt{Hedge}-based mixture of virtual \texttt{LRU} and \texttt{LFU} caches that preserves the \texttt{Hedge} selection probabilities, and hence its regret guarantees, while minimizing the switching cost among all couplings with these marginals.

 \medskip

\noindent \textbf{$\LeCar$ has linear regret.}
This result is first established in Theorem~\ref{thm:linear_regret_lecar} for cache capacity equal to~$2$, and is then extended to arbitrary cache capacities in Theorem~\ref{thm:lecar_linear_regret_general_C}. Establishing this negative result is not immediate. Standard adversarial traces, such as round-robin traces, are ineffective because $\LRU$ and $\LFU$ often recommend the same eviction. Moreover, on stationary or clearly recency friendly traces, eviction feedback tends to reveal the better rule, allowing $\LeCar$ to adapt. Beyond the choice of trace, the analysis is complicated by $\LeCar$'s time-varying selection probabilities: each randomized eviction can lead to a different future cache state, making the algorithm's evolution difficult to track. To handle this difficulty, we construct a periodic request sequence for which $\LRU$ and $\LFU$ incur the same number of misses in each period, but recommend storing different files. Using a negative-drift argument for the Markov chain governing the ratio of their selection probabilities, we show that $\LeCar$ repeatedly switches between the two rules, resulting in linear regret.


\medskip

\noindent \textbf{Caching with experts.}
Rather than using $\LRU$ and $\LFU$ as experts that recommend which item to evict, as in $\LeCar$, we use them as experts that recommend complete cache content. This yields a standard experts problem: at time $t$, $\Hedge$ (for example) outputs a probability $q_t$ of using the cache content of $\LRU$, and probability $1-q_t$ of using $\LFU$. We then use maximal coupling between consecutive expert selections, as in~\cite{daniely_competitive_2019}. The resulting sampler, denoted by $\MC$, preserves the marginals $q_t$ while minimizing the expected number of switches among all selection rules with these marginals (see Cor.~\ref{thm:maximalcoupling}). Combining $\MC$ with $\Hedge$ gives the policy $\HMC$, which has $\cO(\sqrt{T})$ regret and $\cO(\sqrt{T})$ worst-case switching cost, as shown in Cor.~\ref{cor:HMC_regret_switching_cost}. We then show numerically that $\HMC$ outperforms $\LeCar$ on some synthetic traces in Appendix~\ref{app:update_miss_cost}.

\medskip 

\noindent \textbf{Outline.} The rest of the paper is organized as follows. Section~\ref{s:problem} formally defines the problem and describes~$\LeCar$, Section~\ref{s:Lecar_linear_regret} proves that $\LeCar$ suffers linear regret, Section~\ref{s:Caching_experts} introduces the proposed caching policy $\HMC$, and Section~\ref{s:conclusion} concludes the paper.


\section{Problem Formulation}
\label{s:problem}

We consider a server storing a set~$\mathcal{I}$ of $N$ files and a single cache memory, that can store up to~$C<N$ files from~$\mathcal{I}$. A sequence of requests for items in~$\mathcal{I}$ of length~$T$, denoted~$\mathbf{f} = (f_s)_{s\in [T]}$, arrives at the cache. If~$f_t$ is available in the cache at step~$t$, the request is a \textit{hit}, and the cache serves the request. Otherwise, it is a \textit{miss}, and the cache forwards the request to the server.

A caching policy~$\mathcal{A}$ decides the content of the cache at time~$t$, denoted~$S^{\mathcal{A}}_t$, based on past requests $f_1,\ldots ,f_{t-1}$. Algorithm~$\mathcal{A}$ samples~$S^{\mathcal{A}}_t$ from a probability distribution over the space~$\{S\subset \mathcal{I}: |S|=C\}$. Let the total miss count of a caching policy $\mathcal{A}$ until time~$T$ be
\begin{align}\label{Eqmisscount}
    M_{T}^{\mathcal{A}}\triangleq \sum_{t=1}^T \mathbf{1}\{f_t \notin S^{\mathcal{A}}_t\}.
\end{align}

\noindent One of the most popular caching policies are the Least Recently Used ($\LRU$) caching policy that evicts the least-recently used item upon a miss and the Least Frequently Used ($\LFU$) that evicts the least popular item among stored ones upon a miss. In this paper we aim to devise policies that perform as well as the best of $\LRU$ and $\LFU$ for any request sequence. Namely, policies with sublinear regret where the regret is defined as the gap between the miss count of an algorithm and that of the minimum of $\LRU$ and $\LFU$. We denote this metric as $\mathcal{R}_T^{\mathcal{A}}$ and it is formally defined as, 
\begin{align}
\label{e:regret}
\mathcal{R}_T^{\mathcal{A}} \triangleq \sup_{\bm{f}} \left( 
\E{M_{T}^{\mathcal{A}}}-\min\left\{M_{T}^{\LFU},\,M_{T}^{\LRU}\right\}\right),
\end{align}

\noindent \textbf{LeCar.} $\LeCar$~\cite{vietri_driving_2018} was proposed as a caching policy that uses online learning techniques to adaptively balance between $\LRU$ and $\LFU$ updates. The high level idea is that, upon a miss, $\LeCar$ randomize between $\LRU$ and $\LFU$ eviction rules using probabilities derived from scores updated by tracking the history of past evictions. Algorithm~\ref{alg:lecar_update}, in the appendix, describes the caching policy.

At each step~$t$, the algorithm maintains a cache content $S_t\subseteq \mathcal I$, a recency list $R_t$ ordering the cached items from least to most recently used, and LFU hash map $N_t$ for cached items, both of size $C$. In addition, for each expert $e\in\mathcal{E}=\{\LRU,\LFU\}$, the algorithm maintains a weight $w_t^e$ and a history $H_t^e$ of size at most $k$. The history $H_t^e$ is a Hashmap/dictionary that maps previously evicted items by $\LeCar$, on  recommendation by expert $e$, to their eviction times. For example, $H_t^{e}= \{(f,s)\}$ means that file $f$ was evicted by $\LeCar$ on a recommendation by expert~$e$ at step~$s$. We also write $H_{t}^{e}[f]=s$ and we use similar notation for $N_t$.    

At time $t$, the requested item is $f_t$. If $f_t\in S_{t-1}$, the request is a hit. The cache content, experts weights, and histories remain unchanged. The recency list is updated by moving $f_t$ to the most-recently-used position, and the LFU counter of $f_t$ is incremented.

If $f_t\notin S_{t-1}$, the request is a miss and an eviction decision is required. Before choosing the eviction, each expert~$e$ receives the decayed loss
\begin{equation}
    \ell_t^e=\mathds{1}\{f_t\in H_{t-1}^e\}
    d^{\,t-H_{t-1}^e[f_t]},
\end{equation}
where $d\in (0,1)$ is parameter of the algorithm. This penalizes expert $e$ if $f_t$ is in the history of expert $e$, $H_t^{e}$.
The expert weights are then updated as, $w_t^e
    =
    w_{t-1}^e \exp(-\eta \ell_t^e)$, and normalized into probabilities $p_t^e \propto w_t^e$. 

An expert~$E_t$ is sampled according to $\boldsymbol p_t = (p_t^e)_{e\in \mathcal{E}}$ and its eviction recommendation is followed. The $\LRU$ expert recommends the least recently used item, $v_t^{\LRU}=\operatorname{first}(R_{t-1})$, whereas the $\LFU$ expert recommends an item with minimum counter, $v_t^{\LFU} \in \arg\min_{i\in S_{t-1}} N_{t-1}[i]$, with ties broken according to the LRU order. The cache is then updated so that the recommended item for eviction $v_t^{E_t}$ is removed from the cache and the requested item is added instead. 

The recency list removes the evicted item $v_t^{E_t}$ and inserts $f_t$ in the most-recently-used position. Similarly, the counts Hashmap $N_t$ removes $v_t^{E_t}$'s entry and adds $f_t$ with count $1$. Finally, the requested item is removed from both histories, since it is now cached, and the item evicted by the selected expert is inserted into the corresponding history with timestamp $t$. If this history exceeds size $k$, its oldest entry is discarded.

\section{LeCar has linear regret}
\label{s:Lecar_linear_regret}

$\LeCar$ was empirically shown to remain competitive with both $\LRU$ and $\LFU$ on traces that alternate between recency-friendly and frequency-friendly phases~\cite{vietri_driving_2018}. Theorem~\ref{thm:linear_regret_lecar} shows that this is not guaranteed when the separation between the two regimes is less clear. In particular, it constructs a trace on which $\LRU$ and $\LFU$ have comparable performance, \textit{at every time}, while maintaining different cache contents. In this case, $\LeCar$ keeps switching between the two eviction rules, causing it to perform worse than both $\LRU$ and $\LFU$.

\begin{theorem}\label{thm:linear_regret_lecar}
    $\LeCar$ has linear regret, i.e., $\mathcal{R}_{T}^{\LeCar} = \Omega(T)$ when the cache capacity $C=2$ and for any histories length $k\geq1$. 
\end{theorem}
\begin{proof}
The request sequence starts with $(b_0,a,a)$. After that, the request occurs in periodic way, each with $6$ requests. In each phase~$i$, the request sequence is
\begin{align}
(a,a,b_i, c_i, b_i,a),
\end{align}
where $b_i$ and $c_i$ are requests for~$2$ distinct items that only appear at phase $i$. Without loss of generality, we assume that $T-3$ is a multiple of $6$.  

\noindent \textbf{Miss count of $\LFU$.}
Initially, $\LFU$ incurs two misses and stores $\{a,b_0\}$. In phase $1$,
\begin{equation}
\{a,b_0\}
\xrightarrow{a,a,b_1}
\{a,b_1\}
\xrightarrow{c_1}
\{a,c_1\}
\xrightarrow{b_1,a}
\{a,b_1\}.
\end{equation}
and $\LFU$ incurs three misses. Hence, at the end of the phase, it again stores $a$ and one phase item, namely $b_1$. The same argument applies inductively to every phase. Therefore, $M_T^{\LFU}=2+(T-3)/2$.

\medskip 

\noindent \textbf{Miss count of $\LRU$.} The same reasoning applies to $\LRU$. In phase $1$,
\begin{equation}
\{a,b_0\}
\xrightarrow{a,a,b_1}
\{a,b_1\}
\xrightarrow{c_1}
\{b_1,c_1\}
\xrightarrow{b_1,a}
\{a,b_1\},
\end{equation}
and $\LRU$ also incurs three misses. Thus, at the end of each phase, it stores $a$ and the corresponding phase item. By induction, $M_T^{\LRU}=2+(T-3)/2$.
\medskip 

\noindent \textbf{Miss count of $\LeCar$.} Define $\tau_i(j)$ as the time index of request number $j\in [6]$ in phase $i$, i.e., $\tau_i(j) = 3 + 6(i-1)+ j$. Assume that the phase starts from $S^{\LeCar}=\{a,b_{i-1}\}$. The first two requests to $a$ are hits, while the request to the fresh item $b_i$ evicts $b_{i-1}$, so the cache becomes $\{a,b_i\}$. The only nontrivial branching occurs after the request to $c_i$, and possibly after the following request to $b_i$, depending on which expert is selected for eviction. In all possible branches, however, the phase ends again with $S^{\LeCar}=\{a,b_i\}$. We denote the three possible paths during phase~$i$ by
\[
\begin{aligned}
A_i^{1}
&\triangleq \{E_{\tau_i(4)}=\LRU\},\\
 A_i^{2}
&\triangleq \{(E_{\tau_i(4)},E_{\tau_i(5)})=(\LFU,\LRU)\},\\
A_i^{3}
&\triangleq \{(E_{\tau_i(4)},E_{\tau_i(5)})=(\LFU,\LFU)\},
\end{aligned}
\]
where $E_t$ designates the expert selected by $\LeCar$ at step $t$. These three events describe all possible evolutions of $S^{\LeCar}$ inside phase~$i$, as shown in~\eqref{diagram:cache_evolution_LeCar}. 
\begin{equation}\label{diagram:cache_evolution_LeCar}
\begin{aligned}
&\{a,b_{i-1}\}
\xrightarrow{a,a,b_i}
\{a,b_i\}\\
&\xrightarrow{c_i} 
\begin{cases}
\{b_i,c_i\}\xrightarrow{b_i}\{b_i,c_i\}\xrightarrow{a}\{a,b_i\},
& \text{on } A_i^{1},\\
\{a,c_i\}\xrightarrow{b_i}\{b_i,c_i\}\xrightarrow{a}\{a,b_i\},
& \text{on } A_i^{2},\\
\{a,c_i\}\xrightarrow{b_i}\{a,b_i\}\xrightarrow{a}\{a,b_i\},
& \text{on } A_i^{3}.
\end{cases}
\end{aligned}
\end{equation}

Although the quantity $S^{\LeCar}$ exhibits a periodic pattern per phase, the probabilities of the three paths in each phase, depend on the evolution of the expert weights. We therefore track the relative weight of $\LRU$ with respect to $\LFU$. For $i\ge 1$ and $j\in[6]$, define
\begin{align}
\rho_i(j)
\triangleq
\frac{w_{\LRU}(\tau_i(j))}{w_{\LFU}(\tau_i(j))}.
\end{align}
We also write $\rho_i\triangleq \rho_i(1)$ for the value of this ratio at the beginning of phase~$i$. Conditioned on $\rho_i(j)=\rho$, the probability that $\LeCar$ selects $\LRU$ at an eviction time is
\begin{align}
\Proba{E_{\tau_i(j)}=\LRU \mid \rho_i(j)=\rho}
=
\frac{\rho}{1+\rho}.
\end{align}

\noindent Let $\mu_1=\exp(-\eta d)$, and $\mu_2=\exp(-\eta d^2)$. Conditioned on $\rho_i=\rho$, the ratio $\rho_i(j)$ changes only at the steps where one of the experts is penalized. Along the three path events $\mathcal A_i^1,\mathcal A_i^2,\mathcal A_i^3$, its evolution is
\begin{align}\label{e:rho_dynamics}
\rho
\xrightarrow{a,a,b_i,c_i,b_i}
\begin{cases}
\rho
\xrightarrow{a} \mu_2\rho,                             & \text{on }  A_i^1,\\
\frac{\rho}{\mu_1} \xrightarrow{a} \rho,               & \text{on }  A_i^2,\\
\frac{\rho}{\mu_1} \xrightarrow{a} \frac{\rho}{\mu_1}, & \text{on }  A_i^3.
\end{cases}
\end{align}
Therefore,
\begin{align}\nonumber
\Proba{ A_i^1\mid \rho_i}
&= \frac{\rho_i}{1+\rho_i},\\ \nonumber 
\Proba{A_i^2\mid \rho_i}
&= \frac{\rho_i}{(1+\rho_i)(\mu_1+\rho_i)},\\ \label{e:transition_proba_rho}
\Proba{\mathcal A_i^3\mid \rho_i}
&= \frac{\mu_1}{(1+\rho_i)(\mu_1+\rho_i)}.
\end{align}
Note that, even when the history sizes are unbounded, the histories at the end of each phase contain only items that will never be requested again. Therefore, these items do not incur any future penalties, regardless of which history contains which item.

\noindent We denote the expected number of misses per phase~$i$ conditioned on~$\rho_i$ as $m(i|\rho)$:
\begin{align}
        m(i|\rho) = \E{\sum_{j=1}^{6} \mathds{1}\left( f_{\tau_i(j)} \notin S_{\tau_i(j)}^{\LeCar} \right) \mid \rho_i=\rho}
\end{align}
From \eqref{diagram:cache_evolution_LeCar}, $\LeCar$ incurs three misses on $A_i^1$ and $A_i^3$, and four misses on $A_i^2$. Hence, 
\begin{align}
     m(i|\rho) = 3 + \frac{1}{1+\rho} \frac{\rho}{\mu_1 + \rho}
\end{align}
The excess miss probability with respect to $\LRU$ and $\LFU$ in one phase is thus $\frac{1}{1+\rho}\frac{\rho}{\mu_1+\rho}$. This quantity is small only when $\rho$ is close to $0$ or very large, which corresponds to $\LeCar$ selecting essentially only one of the two
experts. Hence, to prove linear regret, it is enough to show that the weight ratio does not escape to either extreme. More precisely, suppose that there exist constants $0<\rho_-<\rho_+<\infty$ and $\kappa >0$, such that for every phase $i$, 

\begin{align}\label{e:sandwich_rho}
\Proba{\rho_i\in[\rho_-,\rho_+]}\ge \kappa.
\end{align}
Then for any phase $i$, 
\begin{align}
\E{m(i|\rho)} \ge
3+
\kappa
\min_{\rho\in[\rho_-,\rho_+]}
\left\{
\frac{1}{1+\rho}\frac{\rho}{\mu_1+\rho}
\right\}.
\end{align}
The minimum is strictly positive because $[\rho_-,\rho_+]$ is a compact subset of $(0,\infty)$. Therefore, $\LeCar$ incurs a constant excess number of misses per phase compared with $\LRU$ and $\LFU$, which implies $\mathcal R_T^{\LeCar}=\Omega(T)$. Next we prove~\eqref{e:sandwich_rho}. 

\medskip

\noindent \textbf{Proof of~\eqref{e:sandwich_rho}.} Combining \eqref{e:rho_dynamics} and \eqref{e:transition_proba_rho} yields that $\rho_i$ is Markov process with the following transition probabilities,
\begin{align}
\rho_{i+1}=
\begin{cases}
\mu_2\rho_i, 
& \text{w.p. } \displaystyle \frac{\rho_i}{1+\rho_i},\\
\rho_i, 
& \text{w.p. } \displaystyle 
\frac{1}{1+\rho_i}\frac{\rho_i}{\rho_i+\mu_1},\\[3mm]
\frac{\rho_i}{\mu_1}, 
& \text{w.p. } \displaystyle 
\frac{1}{1+\rho_i}\frac{\mu_1}{\rho_i+\mu_1}.
\end{cases}
\end{align}

\noindent Let $\gamma_i \triangleq \ln(\rho_i)$ and consider the following conditional expected drift, denoted as $\Delta(\rho)$, and defined as,
\begin{align}
      \Delta(\rho) \triangleq \E{\gamma_{i+1} - \gamma_{i}|\gamma_{i} = \ln(\rho)}.
\end{align}
Define $\beta_1 = \ln(\mu_1)$ and $\beta_2=\ln(\mu_2)$. Direct computations yields,
\begin{align}
      \Delta(\rho) 
&=  \frac{\rho}{1+\rho} \cdot \beta_2 - \frac{1}
    {1+\rho} \cdot \frac{\mu_1}{\mu_1 + \rho} \cdot \beta_1\\
    &=\frac{\beta_2}{(1+\rho)(\mu_1+\rho)}\cdot \left( \rho^2 + \mu_1 \rho - \mu_1 \cdot \frac{\beta_1}{\beta_2}\right).
\end{align}


\noindent Note that $\beta_1= -\eta d$ and $\beta_2 = -\eta d^{2}$ with $d\in (0,1)$, and thus $\beta_2<0$ and $\beta_1/\beta_2= 1/d>1$. Thus, the roots of the quadratic function $\rho^2+\mu_1\rho- \mu_1\frac{\beta_1}{\beta_2}$ are $\frac{-\mu_1\pm\sqrt{\mu_1^2+ 4\mu_1\frac{\beta_1}{\beta_2}}}{2}$. And there exists $\delta>0$, and $0<r_1<r_2$ such that, 
\begin{align}
      \forall \rho \in (0,r_1), \;  \Delta(\rho) > \delta, \text{ and } \forall \rho \in (r_2,+\infty), \;    \Delta(\rho) < - \delta.   
\end{align}
Given that $\gamma_i$ has bounded increments, i.e., $\gamma_{i+1}-\gamma_i
    \in \{\beta_2,0,-\beta_1\}$, \cite[Thm. 1]{pemantle_moment_1999} implies that there exists a constant~$L>0$ such that,  
\begin{align}
    \sup_{i\geq 1} \E{|\gamma_i|} \leq L.
\end{align}
This permits to use Markov inequality to prove that that there exists $\kappa>0$ such that $\gamma_i$ is bounded in a finite segment with a probability at least $\kappa$. This is equivalent to proving~\eqref{e:sandwich_rho} and finishes the proof.
\end{proof}

\begin{theorem}\label{thm:lecar_linear_regret_general_C}
    The statement of Theorem~\ref{thm:linear_regret_lecar} also holds for any cache size $C\geq 2$.
\end{theorem}

\begin{proof}
    The proof follows the same argument as in the case $C=2$, by adding $C-2$ persistent items that remain permanently in the cache without affecting the dynamics of the two remaining cache slots. Let $G=(g_1,\ldots,g_{C-2})$ be a sequence of $C-2$ distinct items. We modify the request sequence in the proof of Thm.~\ref{thm:linear_regret_lecar} so that the initial requests are $G,\,G,\,b_0,\,a,\,a$, ensuring that $S^{\LeCar}= \{ a, b_0, G\}$ at the beginning. Then in each phase $i$, the request sequence is,
\begin{align}
 (a,a,G,b_i,G,c_i,G,b_i,G,a).
\end{align}

The cache-state evolution on the two non persistent slots has the same structure as in the proof of Theorem~\ref{thm:linear_regret_lecar}. The only change is that the elapsed times between an eviction and the corresponding future request are larger because of the intervening requests to $G$. This only replaces $\mu_1$ and $\mu_2$ by constants $\mu_1^{(C)},\mu_2^{(C)}\in(0,1)$, and the same drift argument applies. 
\end{proof}

\section{Caching with Experts}
\label{s:Caching_experts}

To obtain a caching policy with sublinear regret guarantees, in the sense of~\eqref{e:regret}, we cast the problem as an experts problem. The set of experts is $\cE = \{\LRU,\LFU\}$. At each time step~$t$, each expert recommends the cache state that would be maintained by its corresponding virtual cache, namely $S_t^{\LRU}$ and $S_t^{\LFU}$. The loss of expert $e\in\cE$ at time~$t$ is defined as $\ell_t^e = \mathds{1}\{f_t \notin S_t^e\}$. At each step $t$, the algorithm $\cA$ selects an expert $E_t^{\cA}$. We denote its probability of selecting $\LRU$ as $q_t^{\cA} \triangleq \Proba{E_t^{\cA} = \LRU}$. The expected loss of~$\cA$ at time~$t$ is therefore linear in its decision,
\begin{align}
    \ell_t(\cA) = q_t^{\cA}\ell_t^{\LRU} + \left(1-q_t^{\cA}\right)\ell_t^{\LFU}.
\end{align}
With this choice of experts, decision set, and loss function, the regret of the resulting online decision problem coincides with the regret defined in~\eqref{e:regret}. Applying the Hedge algorithm in this setting yields
\begin{align}\label{e:hedge_update}
    &S_t^{\texttt{Hedge}}
    =
    \begin{cases}
        S_t^{\LRU}, & \text{w.p. }
        \psi\left(M_{t-1}^{\LRU}-M_{t-1}^{\LFU}\right), \\
        S_t^{\LFU}, & \text{otherwise,}
    \end{cases}
\end{align}
where $\psi(x) \triangleq \frac{1}{1+\exp(\eta x)}$. Equivalently, $q_t^{\texttt{Hedge}}= \psi\left(\Delta_t\right)$, where $\Delta_t=M_{t-1}^{\LRU}-M_{t-1}^{\LFU}$, and $M_{t-1}^{\LRU}$ and $M_{t-1}^{\LFU}$ denote the miss count of $\LRU$ and $\LFU$ over the first $t-1$ requests. Taking $\eta=\Theta(1/\sqrt{T})$ yields asymptotically optimal sublinear regret: $\mathcal{R}_T^{\texttt{Hedge}}= \cO(\sqrt{T})$. Appendix~\ref{app:Hedge} provides more details about the \texttt{Hedge} algorithm in general.

Despite its optimal regret guarantee, using \texttt{Hedge} to mix $\LRU$ and $\LFU$ may incur a large upload cost. Whenever $E_{t+1}^{\cA}\neq E_t^{\cA}$, the algorithm may need to replace its current cache state with that of the newly selected policy. In the worst case, the two virtual caches are disjoint, requiring the upload of as many as~$C$ files. By contrast, when $E_{t+1}^{\cA}=E_t^{\cA}$, the cache performs a standard $\LRU$ or $\LFU$ update, which requires at most one upload. Thus, switches capture the main source of excess update cost, motivating our focus on the expert-switching cost.

\begin{align}
        U_T^{\cA}(\bm f) \triangleq \sum_{t=1}^{T-1} \Proba{E_{t+1}^{\cA}\neq E_t^{\cA}},
\end{align}
which counts the expected number of expensive transitions between the two virtual caches. Thus, besides the marginal probabilities $q_t^{\cA} \triangleq \Proba{E_t^{\cA}=\LRU}$, an algorithm is also characterized by the joint distribution of $(E_t^{\cA})_{t\in[T]}$.

There is an extensive literature on reducing switching costs in the experts problem, including lazy variants of \texttt{Hedge} and \texttt{Follow the Perturbed Leader}~\cite{geulen_regret_2010,kalai_efficient_2005,devroye_prediction_2013}. More directly relevant to our construction,~\cite{daniely_competitive_2019} uses maximal coupling between consecutive expert-selection distributions in a regret formulation that includes switching costs. For two experts, this construction reduces to maximal coupling of two Bernoulli random variables~\cite[Chapter~1.4]{thorisson_coupling_2000}. We state the corresponding coupling property in Claim~\ref{claim:maximal_coupling_Bernoulli} of Appendix~\ref{app:proofs}. Specifically, the maximal-coupling sampler, denoted by $\MC$, satisfies
\begin{align}\nonumber 
\Proba{E^{\MC}_{t+1}=\LFU \mid E^{\MC}_{t}=\LFU} &= \frac{\min\{1-q_t,1-q_{t+1}\}}{1-q_t}, \\ \label{e:MC_sampler}
\Proba{E^{\MC}_{t+1}=\LRU \mid E^{\MC}_{t}=\LRU} &= \frac{\min\{q_t,q_{t+1}\}}{q_t},
\end{align}

\noindent Corollary~\ref{thm:maximalcoupling} shows that this construction preserves the marginals, and minimizes the switching cost. 
\begin{corollary}\label{thm:maximalcoupling}
For any request sequence $\bm f$ of length $T$ and any sequence of probabilities $(q_t)_{t\in[T]}$, 
\begin{align}
        \min_{\cA:\;\forall t, \; q_t^{\cA}=q_t} U_T^{\cA}(\bm f)= \sum_{t=1}^{T-1} |q_{t+1}-q_t|.
\end{align}
Moreover, this minimum is achieved by the maximal coupling sampler $\MC$ in~\eqref{e:MC_sampler}. 
\end{corollary}

\begin{proof}[of Cor.~\ref{thm:maximalcoupling}]
Using Claim~\ref{claim:maximal_coupling_Bernoulli} for two Bernoulli random variables, $E_t$ and $E_{t+1}$, with success probabilities $q_t$ and $q_{t+1}$, we deduce that $(E_t^{\MC},E_{t+1}^{\MC})$ is a maximal coupling of these two distributions, i.e., $\Proba{E_t^{\MC} \neq E_{t+1}^{\MC}} = |q_{t+1}- q_{t}|$. Thus $\MC$ minimizes the switching cost.
\end{proof}

\begin{algorithm}[t]
\caption{\texttt{H-MC}: Hedge with maximal-coupling}
\label{alg:HMC}
\begin{algorithmic}[1]
\Require Request sequence $\boldsymbol f=(f_1,\dots,f_T)$, cache size $C$, learning rate $\eta>0$
\State $(\Delta_{0},q_1)\gets (0,0.5)$
\State Sample $E_1\sim \mathrm{Bernoulli}(q_1)$, where $E_1=\LRU$ w.p. $q_1$

\For{$t=1$ to $T$}
    \State $S_t\gets S_t^{E_t}$
    \State $\delta_t \gets \mathds{1}\{f_t\notin S_t^{\LRU}\} - \mathds{1}\{f_t\notin S_t^{\LFU}\}$
    \State $\Delta_t\gets \Delta_{t-1} + \delta_t$
    \State $q_{t+1}\gets \psi\left( \Delta_t \right)$ 
    \If{$\left(\delta_t= 1,E_t=\LRU\right)$}
            \State $E_{t+1}\gets \LFU$ w.p. $1-\frac{q_{t+1}}{q_t}$
    \ElsIf{$\left(\delta_t= -1, E_t=\LFU\right)$}
            \State $E_{t+1}\gets \LRU$ w.p. $\frac{q_{t+1}-q_t}{1-q_t}$
    \Else
        \State $E_{t+1}\gets E_t$
    \EndIf
\EndFor
\end{algorithmic}
\end{algorithm}

Combining \texttt{Hedge} with $\MC$ yields a caching algorithm that we call \texttt{H-MC}. This policy is a lazy implementation of \texttt{Hedge} that minimizes the switching cost. Algorithm~\ref{alg:HMC} describes an implementation of~\texttt{H-MC}. More precisely, after observing request $f_t$, the algorithm computes
\begin{align}
 \delta_t = \mathds{1}\{f_t\notin S_t^{\LRU}\}-\mathds{1}\{f_t\notin S_t^{\LFU}\},
\end{align}
and sets $q_{t+1}$ as the probability of sampling an expert in \texttt{Hedge}, i.e., $q_{t+1}=\psi(\Delta_t)$, with $\Delta_t= M_{t}^{\LRU}-M_{t}^{\LFU}$.

If $\delta_t=0$, both experts incur the same loss, and therefore $q_{t+1}=q_t$. In this case the maximal coupling keeps the same expert with probability one. If $\delta_t=1$, then $\LRU$ misses while $\LFU$ hits, so the \texttt{Hedge} probability of $\LRU$ decreases, i.e., $q_{t+1}<q_t$. Hence, only when the current expert is $\LRU$ may a switch be needed; the algorithm sets $E_{t+1}^{\MC}=\LFU$ w.p. $1-\frac{q_{t+1}}{q_t}$. Similarly, if $\delta_t=-1$, then $\LRU$ hits while $\LFU$ misses, so $q_{t+1}>q_t$. In this case, only when the current expert is $\LFU$ may a switch be needed. The algorithm then sets $E_{t+1}^{\MC}=\LRU$ w.p. $\frac{q_{t+1}-q_t}{1-q_t}$.

Corollary~\ref{cor:HMC_regret_switching_cost} shows that \texttt{H-MC} has sublinear regret and a worst-case switching cost of~$\cO(\sqrt{T})$.

\begin{corollary}\label{cor:HMC_regret_switching_cost}
If $\eta=\Theta(1/\sqrt{T})$, then $\mathcal{R}_{T}^{\HMC}= \cO(\sqrt{T})$ and $U_T^{\HMC}(\bm{f})=\cO(\sqrt{T})$, for any requests $\bm{f}$. 


\end{corollary}

\begin{proof}
 Using Thm.~\ref{thm:maximalcoupling}, $\MC$ is maximal coupling and thus it preserves the marginal probabilities, i.e.,\\
 $\Proba{E_t^{\MC} = \LRU}=~q_t^{\texttt{Hedge}}$. It follows that $\HMC$ inherits the sublinear regret guarantees of $\Hedge$, i.e., $\mathcal{R}_{T}^{\HMC}= \cO(\sqrt{T})$. The same theorem also computes the corresponding switching cost. Moreover, 
 \begin{align}\nonumber 
        |q_{t+1}^{\Hedge} - q_t^{\Hedge}| 
&= |\psi(\Delta_t) - \psi(\Delta_{t-1})| \\ 
&\leq \frac{\eta}{4} |\Delta_t - \Delta_{t-1}| \leq \frac{\eta}{4}.
 \end{align}
 We used above the fact that the derivative of $\psi$ is equal to $\eta \exp(\eta x)/(1+\exp(\eta x))^{2}$, which is smaller than $\eta/4$. This proves that the switching cost is $\cO(\sqrt{T})$ for any $\bm{f}$.
\end{proof}



To further reduce the switching cost, one may also use simple schemes such as batching decisions, i.e., resampling the expert only once every $B$ requests, combined with sampling the request sequence with some probability $p$. Although such schemes preserve sublinear regret, they worsen the regret bound by a factor~$\cO(\sqrt{B/p})$~\cite{ben_mazziane_efficient_2026}. Appendix~\ref{app:update_miss_cost} evaluates the upload cost and the miss count of $\HMC$ when augmented with these modifications, and shows that $\HMC$ outperforms $\LeCar$ on some synthetic traces.

Finally, although $\HMC$ minimizes the switching cost, it may still be undesirable to incur a large upload cost in a single step when the selected expert changes. To mitigate this issue, observe that any algorithm $\cA$ satisfying
\begin{align}\label{Eqmarginal}
    \Proba{ i\in S^{\cA}_t }=
    \begin{cases}
        1, & i \in S^{\LRU}_{t} \cap S^{\LFU}_{t},\\
        q_t^{\Hedge}, & i\in S^{\LRU}_{t} \setminus S^{\LFU}_{t},\\
        1-q_t^{\Hedge}, & i\in S^{\LFU}_{t} \setminus S^{\LRU}_{t},\\
        0, & \text{otherwise},
    \end{cases}
\end{align}
has the same expected loss as \texttt{Hedge} in \eqref{e:hedge_update}. This observation allows for more flexible implementations. For example, instead of selecting one entire virtual cache, one can construct a randomized cache by coupling the files in $S_t^{\LRU}\setminus S_t^{\LFU}$ with those in $S_t^{\LFU}\setminus S_t^{\LRU}$ and storing exactly one file from each pair, with probabilities $q_t^{\Hedge}$ and $1-q_t^{\Hedge}$, respectively. Such an implementation keeps the same expected regret guarantee as \texttt{Hedge}, while avoiding the need to upload an entire virtual cache in one shot.



%

\section{Conclusion}
\label{s:conclusion}

We showed that the $\LeCar$-style approach to combining $\LRU$ and $\LFU$ has linear regret with respect to the better of the two policies. We then give an explicit Hedge based mixture of virtual LRU and LFU caches using an established maximal-coupling construction. This implementation preserves Hedge’s selection probabilities and minimizes the expected number of expert switches among all selection rules with these marginals. Future work will investigate caching policies that retain sublinear regret guarantees while accounting for the number of uploaded files and avoiding large uploads concentrated in a single step.

\section{Acknowledgments}

The authors thank the anonymous reviewers of IFIP Performance for their insightful comments, which helped improve the paper. YB also thanks Isidoor Pinillo Esquivel for a helpful discussion during a visit to Inria, and Giovanni Neglia and Sara Alouf for hosting the visit. XZou is supported by ELLIIT and the Knut and Alice Wallenberg Foundation.


\onecolumn 
\appendix

\section{LeCar}
\begin{algorithm}
\algrenewcommand\algorithmicrequire{\textbf{Input:}}
\algrenewcommand\algorithmicensure{\textbf{Output:}}
\caption{LeCaR-style update with LRU and LFU experts}
\label{alg:lecar_update}
\begin{algorithmic}[1]
\Require Request sequence $\boldsymbol f=(f_1,\dots,f_T)$, cache size $c$, initial cache content $S_0$, initial recency list $R_0$, learning rate $\eta>0$, discount factor $d\in(0,1]$, history size $k$
\Ensure Cache states $S_1,\ldots, S_T$. 
\State $\mathcal E\gets\{\LRU,\LFU\}$
\State $N_0\gets \{ (i,0): \; i\in S_0 \}$, \Comment{Initial $\LFU$ counters}
\State $(w_0^{e}, H_0^e) \gets (1, \emptyset ) \; \forall e\in \mathcal{E} $ \Comment{Initial expert weight and history}

\For{$t=1$ to $T$}
\Comment{Increment only the requested item's $\LFU$ counter}
    \If{$f_t\in S_{t-1}$}
        \State $(S_t, w_t^e, H_t^{e})\gets (S_{t-1}, w_{t-1}^e, H_{t-1}^{e})$
        \State $R_t\gets \operatorname{MRU}(R_{t-1},f_t)$ \Comment{Move $f_t$ to most recent position}
        \State $N_t[i]\gets N_{t-1}[i]+\mathds{1}\left( i=f_t\right)$, $\forall i\in N_{t-1}$ 
        
    \Else
        
        \State $\ell_t^e \gets  \mathds{1}\left(f_t\in H_{t-1}^e\right) d^{t-H_{t-1}^e[f_t]}$, for all $e$
        \State $w_t^e\gets w_{t-1}^e\exp(-\eta \ell_t^e)$, $\forall e$
        \Comment{Penalize the expert whose history contains $f_t$}
        \State $p_t^e\gets \frac{w_t^e}{\sum_{e'\in\mathcal E} w_t^{e'}}$, $\forall e$
        \State Sample $E_t\sim \boldsymbol{p}_t$
        \Comment{Choose eviction expert}
         \State $v_t^{\mathrm{LRU}}\gets \operatorname{first}(R_{t-1})$
        \Comment{$\LRU$ recommendation}
        \State $v_t^{\mathrm{LFU}}\gets
        \arg\min_{i\in S_{t-1}}
        \left(N_{t-1}(i)\right)$ (ties broken by $\LRU$)
        \Comment{$\LFU$ recommendation}
       
        \State $S_t\gets S_{t-1}\setminus\{v_t^{E_t}\}\cup\{f_t\}$
        \Comment{Evict $v_t^{E_t}$ and insert $f_t$}

        \State $R_t\gets \mathrm{MRU}(R_{t-1}\setminus\{v_t^{E_t}\},f_t)$
        \Comment{Update recency list}
        \State $N_t \gets N_{t-1} \cup \{ (f_t,1) \} \setminus \{ v_t^{E_t}, N_{t-1}[v_t^{E_t}]\}$. 
        \State $H_t^e\gets H_{t-1}^e\setminus\{f_t\}$, $\forall e$
        \Comment{Remove requested item from the history}
        \State $H_t^{E_t}\left[v_t^{E_t}\right] =t$  
        \Comment{Insert evicted item with timestamp $t$}
         \State Evict from $H_t^{E_t}$ the item with the oldest timestamp when $|H_t^{E_t}|\geq k$

    \EndIf
\EndFor
\end{algorithmic}
\end{algorithm}

\section{Maximal coupling construction}
\label{app:proofs}

\begin{claim}\label{claim:maximal_coupling_Bernoulli}
Let $X$ and $Y$ be Bernoulli random variables with success probabilities $x$ and $y$. Among all joint distributions with these marginals, the minimum possible mismatch probability, $\Proba{X\neq Y}$, is $|x-y|$. Moreover, this value is achieved by the maximal coupling, for which
\begin{align}\nonumber 
        \Proba{X=1,Y=1} &= \min\{x,y\},\\
        \Proba{X=0,Y=0} &= \min\{1-x,1-y\}.
\end{align}
\end{claim}

\begin{proof}[of Claim~\ref{claim:maximal_coupling_Bernoulli}]
Let $p$ and $q$ denote the distributions of $X$ and $Y$ on $\{0,1\}$, i.e., $p(1)=x$, $p(0)=1-x$, $q(1)=y$, and $q(0)=1-y$. Define $r(i) \triangleq \min\{p(i),q(i)\}$, $\forall i \in \{0,1\}$ and $ \alpha \triangleq r(0)+r(1)$. \cite[Thm. 4.1]{thorisson_coupling_2000} implies that $1- \alpha= |x-y|$ is the minimum possible mismatch probability.  

We now use the maximal coupling construction from~\cite[Chap. 1, Sec.4]{thorisson_coupling_2000}. Let $I$ be a Bernoulli random variable with parameter $\alpha$. If $I=1$, $X=Y=V$, where $V$ in a random variable on $\{0,1\}$ with distribution
\begin{align}
 \Proba{V=i}=\frac{r(i)}{\alpha}, \forall i\in \{0,1\}.
\end{align}
Otherwise, $X=W_X$ and $Y=W_Y$ such that the distributions of $W_X$ and $W_Y$ are given by, 
\begin{align}
        \Proba{W_X=i} = \frac{p(i)-r(i)}{1-\alpha}
       ,\; 
        \Proba{W_Y=i}= \frac{q(i)-r(i)}{1-\alpha}.
\end{align}
\cite[Thm. 4.2]{thorisson_coupling_2000} shows that this construction minimizes the mismatch probability, i.e., $\Proba{X\neq Y} = |x-y|$. 

Note that $W_X$ and $W_{Y}$ have disjoint supports, since for each $i\in\{0,1\}$, at most one of $p(i)-r(i)$ and $q(i)-r(i)$ is positive. Hence, conditioned on $I=0$, we have $X\neq Y$. Therefore,
\begin{align}
\Proba{X=Y=1}=\Proba{I=1, V=1}= \frac{\alpha r(i)}{\alpha} = \min(x,y). 
\end{align}
Similar arguments hold for $X=Y=0$, which finishes the proof. 
\end{proof} 

\section{Hedge} 
\label{app:Hedge}
In the standard experts problem, each expert \(e\in\mathcal{E}\) incurs a loss \(\ell_t^e\in[0,1]\) at time \(t\). Let \(M_t^e=\sum_{s=1}^t \ell_s^e\) denote its cumulative loss. The \texttt{Hedge} algorithm selects expert \(e\) at time \(t\) with probability
\[
p_t^e
=\frac{\exp(-\eta M_{t-1}^e)}
        {\sum_{e'\in\mathcal{E}}\exp(-\eta M_{t-1}^{e'})},
\]
and therefore incurs expected loss \(\ell_t(\texttt{Hedge})=\sum_{e\in\mathcal{E}}p_t^e\ell_t^e\). For any loss sequence, its regret satisfies
\[
R_T^{\texttt{Hedge}}
\triangleq
\sum_{t=1}^T \ell_t(\texttt{Hedge}) -\min_{e\in\mathcal{E}}M_T^e
\leq \frac{\log|\mathcal{E}|}{\eta}+\frac{\eta T}{8}.
\]
Thus, choosing \(\eta=\sqrt{8\log|\mathcal{E}|/T}\) gives \(R_T^{\texttt{Hedge}}=\mathcal{O}(\sqrt{T\log|\mathcal{E}|})\), which is minimax optimal up to constant factors~\cite[Thm.~2.2 and Sec.~3.7]{cesa-bianchi_prediction_2006}.

\section{Simulations: Upload cost and Miss ratio}
\label{app:update_miss_cost}
We compare the miss ratio of \texttt{H-MC} with existing policies $\LeCar$, $\LRU$, $\LFU$ and \texttt{Hedge} (independent sampling of the expert at each step), using synthetic traces. To quantify the benefit of maximal coupling, we compare the upload costs of \texttt{H-MC} and \texttt{Hedge}, where the upload cost of policy $\mathcal{A}$ is measured by the average number of new files that must be replaced into the cache between consecutive time steps, denoted by
\begin{equation}
    C^{\mathcal{A}}_T = \frac{1}{T}\sum_{t=1}^{T-1}|S^{\mathcal{A}_{t+1}} \backslash S^{\mathcal{A}}_t |.
\end{equation}

\subsection{Traces}

\textbf{Zipf trace.} Requests are generated independently from a catalog of $N=1000$ items according to a Zipf distribution with exponent $\alpha=1.1$. More precisely, $\Proba{f_t = i} \propto i^{-\alpha}$.
The trace contains $T=100,000$ requests.

\textbf{Scan/churn trace.} This non-stationary trace alternates between sequential scans and churns over changing working sets. In each cycle, the scan contains $3C$ (cache size $C$) previously unseen objects, requested once each. It is followed by a churn phase containing $2C$ new objects, which are requested repeatedly four times in randomly shuffled order. Thus, for $C=100$, each cycle contains $1,100$ items. We use $100$ cycles, yielding $T=110,000$.

\textbf{Markov-modulated trace} Requests alternate between frequency and recency oriented regimes according to a symmetric two-state Markov chain $Z_T\in \{0,1\}$ with $\Proba{Z_{t+1}=1 | Z_{t}=0}=\Proba{Z_{t+1}=0 | Z_{t} = 1} = 1-p_s$, where $p_s = 0.002$. Hence, each regime persists for an average of $\frac{1}{p_s} =500$ requests. Conditional on $Z_t=0$, requests follow a fixed Zipf distribution with exponent $\alpha=1.1$ over $20C = 2000$ items. Conditional on $Z_t=1$, requests are uniformly drawn from a set of $C=100$ items. At each request in this regime, the active set is replaced by a new set of $C$ objects with probability $0.03$, creating short-term temporal locality and making previously accumulated popularity information less useful. The resulting trace therefore alternates between a stationary, frequency-oriented workload and a dynamically changing, recency-oriented workload. We use $T=100,000$ requests.

\textbf{Adversarial trace.} We also evaluate the periodic adversarial construction of Theorem 3.2 for cache capacity $C=8$. The construction contains $C-2$ persistent objects together with the two dynamically evolving cache positions used in the counterexample. We use $2000$ phases, resulting in a trace of $T=60015$ requests.

\textbf{Perturbed Adversarial trace.} To assess robustness of the counterexample to random perturbations, we independently replace each request in the adversarial construction with probability $\epsilon=0.02$. A replacement is drawn uniformly from a background pool of $4C=32$ objects. The phase-specific objects of the original construction remain fresh across phases. This produces a noisy version of the theoretical trace while preserving its underlying phase structure. This trace allows us to test whether the  behavior of $\LeCar$ identified theoretically persists when the adversarial construction is subject to small random perturbations and repeated content. We use $2000$ phases, resulting in a trace of $T=60015$ requests.

\begin{table*}[h]
\centering
\caption{Average miss ratio (\%). }
\label{tab:baseline-miss}
\setlength{\tabcolsep}{5pt}
\begin{tabular}{l c c c c c c c}
\toprule
Workload
& $C$
& \texttt{LRU}  
& \texttt{LFU} 
& \texttt{LeCaR} 
& \texttt{H-MC} 
& \texttt{Hedge} 
\\
\midrule

Zipf
& 100
& 32.5990
& \textbf{26.6590}
& 31.5630
& 26.7724
& 26.7495
\\

Markov
& 100
& \textbf{60.2090}
& 67.7440
& 60.1229
& 60.3245
& 60.3035
\\

Scan/churn
& 100
& \textbf{91.5018}
& 99.8164
& 91.5028
& 91.6440
& 91.5921
\\

Adversarial
& 8
& \textbf{10.0108}
& \textbf{10.0108}
& 10.7487
& 10.0180
& 10.0051
\\

Perturbed adversarial
& 8
& 14.1398
& \textbf{11.7571}
& 12.8473
& 11.8584
& 11.8739
\\
\bottomrule
\end{tabular}
\end{table*}

\begin{figure}[H]
    \centering
    \includegraphics[width=0.8\linewidth]{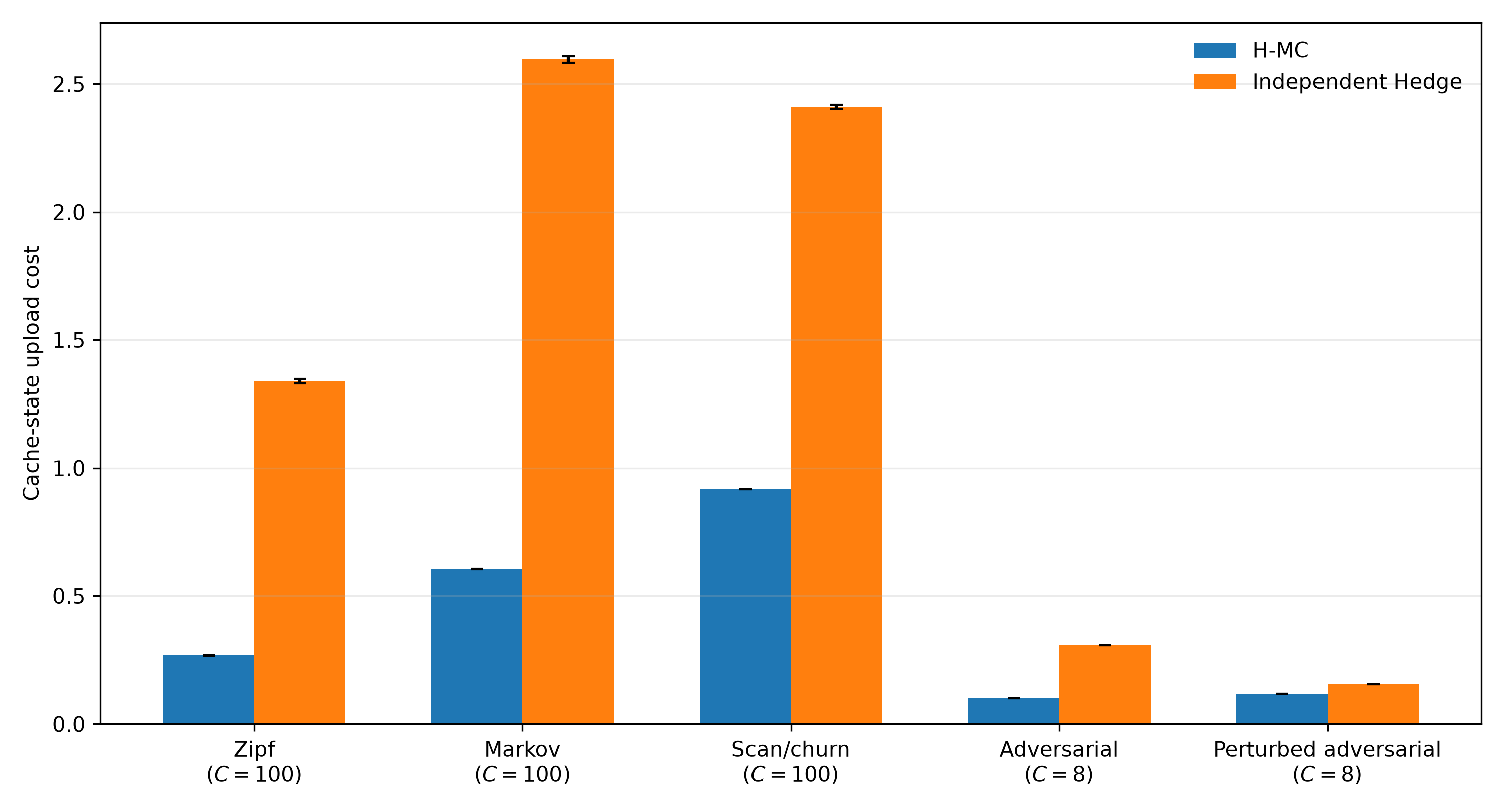}
    \caption{Upload Cost of \texttt{H-MC} and \texttt{Hedge}}
    \label{FiguploadCost}
\end{figure}

\subsection{Experiment Setup and Results}

The baseline evaluation includes five synthetic workloads. $\LRU$ and $\LFU$ are deterministic, whereas $\LeCar$, \texttt{H-MC}, and \texttt{Hedge} are evaluated over 30 random seeds. For \texttt{H-MC} and \texttt{Hedge}, the learning rate is $\eta=\sqrt{8\log 2/T}$. For $\LeCar$, we use $\eta=0.45$, history size $C$, and discount factor $d=0.005^{1/C}$; the same configuration used in the original paper~\cite[Section 3]{vietri_driving_2018}.

Table~\ref{tab:baseline-miss} reports the average miss-ratio of the considered algorithms under the five traces. \texttt{H-MC} has a lower average miss ratio than the evaluated $\LeCar$ implementation on Zipf and both adversarial constructions. However, $\LeCar$ has a lower average miss ratio on Markov-modulated and scan/churn workloads.



Figure~\ref{FiguploadCost} compares the upload cost of \texttt{H-MC} and \texttt{Hedge}. Error bars show $95\%$ normal-approximation confidence intervals across randomized runs. Across the five traces, the maximal coupling feature yields savings in the update cost ranging from $23.6\%$ to $80.0\%$. At the same time, the two methods exhibit nearly identical miss performance: the largest observed difference between their average miss ratios is $0.052\%$. This is consistent with \texttt{H-MC} preserving the \texttt{Hedge} marginal expert-selection probabilities while changing only the coupling between consecutive selections.


\bibliography{references}

\end{document}